\documentclass[11pt, a4paper]{article}

\usepackage[a4paper]{geometry}
\usepackage[english]{babel}
\usepackage{bm}

\usepackage{diagbox}

\usepackage[utf8]{inputenc}
\usepackage{xspace}
\usepackage{amsmath,amsthm,amssymb,mathtools}
\usepackage{url}
\usepackage{dsfont,bbm}
\usepackage{natbib}
\usepackage{algorithm}
\usepackage[noend]{algpseudocode}
\usepackage{xcolor}
\usepackage{tikz}
\usepackage{graphicx}
\usepackage{subcaption}
\usepackage{multirow}
\usepackage{diagbox}
\usepackage{tabularx}
\newcolumntype{Y}{>{\centering\arraybackslash}X}

\newcounter{phase}[algorithm]
\newlength{\phaserulewidth}
\newcommand{\setphaserulewidth}{\setlength{\phaserulewidth}}

\makeatother

\setphaserulewidth{.5pt}

\renewcommand{\labelenumi}{(\alph{enumi})}
\renewcommand\theenumi\labelenumi

\newtheorem{lemma}{Lemma}
\newtheorem{theorem}{Theorem}

\renewenvironment{proof}
{\begin{trivlist}\item\textbf{Proof.}}
{\hspace*{\fill}$\Box$\end{trivlist}}

\newenvironment{proofof}[1]
{\begin{trivlist}\item\textbf{Proof of #1.}}
{\hspace*{\fill}$\Box$\end{trivlist}}

\newcommand{\ooea}{(1+1)~EA\xspace}
\newcommand{\asymsa}{Asym-SA-\ooea}

\newcommand{\om}{\textsc{OneMax}\xspace}
\newcommand{\zeromax}{\textsc{ZeroMax}\xspace}

\newcommand{\onemax}{\om}

\DeclareMathOperator{\Prob}{Pr}

\newcommand{\ones}[1]{\lvert #1\rvert_1}
\newcommand{\zeros}[1]{\lvert #1\rvert_0}

\newcommand{\ie}{i.\,e.\xspace}
\newcommand{\eg}{e.\,g.\xspace}

\newcommand{\Wlo}{W.\,l.\,o.\,g.\xspace}

\newcommand{\prob}[1]{\mathord{\Prob}\mathord{\left(#1\right)}}
\newcommand{\probtext}[1]{\Prob(#1)}

\title{Evolutionary Algorithms with \\Self-adjusting Asymmetric Mutation}

\author{
  Amirhossein Rajabi\\
  Technical University of Denmark \\
	Kgs. Lyngby \\
	Denmark \\
amraj@dtu.dk \\
  \and
    Carsten Witt\\
  Technical University of Denmark \\
	Kgs. Lyngby \\
	Denmark \\
cawi@dtu.dk \\
}

\begin{document}

\maketitle

\begin{abstract}
Evolutionary Algorithms (EAs) and other randomized search heuristics are often considered as unbiased algorithms that are invariant with respect to different transformations of the underlying search space. However, if a certain amount of domain knowledge is available the use of biased search operators in EAs becomes viable. We consider a simple (1+1) EA for binary search spaces and analyze an asymmetric mutation operator that can treat zero- and one-bits differently. This operator extends  previous work by Jansen and Sudholt (ECJ 18(1), 2010) by allowing the operator asymmetry to vary according to the success rate of the algorithm. Using a self-adjusting scheme that learns an appropriate degree of asymmetry, we show improved runtime results on the class of functions OneMax$_a$ describing the number of matching bits with a fixed target $a\in\{0,1\}^n$.

\end{abstract}

\section{Introduction}
The rigorous runtime analysis of randomized search heuristics, in particular evolutionary algorithms (EAs), is a vivid research area \citep{NeumannW10, Jansen13, DoerrNeumannBook2020} that has provided proven results on the efficiency of different EAs in various optimization scenarios and has given theoretically guided advice on the choice of algorithms and their parameters. 
A common viewpoint is that EAs in the absence of problem-specific knowledge should satisfy some \emph{invariance} properties; \eg, for the space $\{0,1\}^n$ of bit strings of length~$n$ it is desirable that the stochastic behavior of the algorithm does not change if the bit positions are renamed or the meaning of zeros and ones at certain bit positions is exchanged (formally, these properties can be described by automorphisms on the search space). Black-box complexity theory of randomized search heuristics (\eg, \cite{DoerrBlackBox2020}) usually assumes such invariances, also known under the name of unbiasedness \citep{LehreWittAlgorithmica12}. 

In a given optimization scenario, a certain amount of domain knowledge might be available that invalidates the unbiasedness assumption. For example, on bit strings it might be known beforehand that zero-bits have a different interpretation than one-bits and that the total number of one-bits in high-quality solutions should lie within a certain interval. A prominent example is the minimum spanning tree (MST) problem \citep{NeumannWegenerTCS07}, where, when modelled as a problem on bit strings of length~$m$, $m$ being the number of edges, all valid solutions should have exactly $n-1$ one-bits, where $n$ is the number of vertices. In fact, the paper \cite{NeumannWegenerTCS07} is probably the first 
to describe a theoretical runtime analysis of an EA with a biased (also called asymmetric) mutation operator: they investigate, in the context of a simple \ooea, an operator that flips each zero-bit of the current bit string~$x$
with probability $1/\zeros{x}$, where $\zeros{x}$ is the number of zero-bits in~$x$, and each one-bit with probability $1/\ones{x}$, where $\ones{x}$ is the number of one-bits. As a consequence, the expected number of zero- and one-bits is not changed by this operator. Different
biased mutation operators have been studied, both 
experimentally and theoretically, in, \eg,
\cite{RaidlKJIEEETEC06, SuttonAlgorithmica16, NeumannANEvoMusArt17}.

The \ooea with asymmetric mutation operator (for short, asymmetric \ooea) proposed in \cite{NeumannWegenerTCS07} was revisited in depth by Jansen and Sudholt \citep{JansenSudholtECJ10} who investigated its effectiveness on a number of problems on bit strings, including the famous \om problem. In particular, they showed the surprising result that the asymmetric \ooea optimizes \om in expected time $O(n)$, while still being able to optimize 
in expected time $O(n \log n)$ 
all functions from the generalized class $\om_a(x)\coloneqq 
n-H(x,a)$, where $a\in\{0,1\}^n$ is an arbitrary target and $H(\cdot,\cdot)$ denotes the Hamming distance. However, the question is open whether this 
speed-up by a factor $\Theta(\log n)$ for the all-ones target (and the all-zeros target) compared to general targets~$a$ is the best possible that can be achieved with an asymmetric mutation. In principle, since the operator knows which bits are ones and which are zeros it is not unreasonable to assume that it could be modified to let the algorithm approach the all-ones string faster than~$\Theta(n)$ -- however,  this must not result in an algorithm that is tailored to the all-ones string and fails badly for other targets.

In this paper, we investigate whether the bias of the asymmetric operator from \cite{JansenSudholtECJ10} 
can be adjusted to put more emphasis on one-bits when it is working on the $\om_{1^n}$ function with all-ones string as target (and accordingly for 
zero-bits with $\om_{0^n}$) while still being competitive with the original operator on general $\om_a$. Our approach 
is to introduce a \emph{self-adjusting bias}: the 
probability of flipping a one-bit is allowed 
to decrease or increase by a certain amount, which is upper bounded by~$r/\ones{x}$ for a parameter~$r$; the probability of 
flipping a zero-bit is accordingly adjusted in the 
opposite direction. A promising setting for this bias is learned through a self-adjusting scheme being similar in style with the 1/5-rule \citep{DoerrDoerrAlgorithmica18} and related 
techniques: in a certain observation phase of length~$N$ two different parameter values are each tried 
$N/2$ times and the value that is relatively more successful is used in the next phase. Hence, this 
approach is in line with a recent line of theoretical research of self-adjusting algorithms where the
concrete implementation of self-adjustment is an ongoing debate \citep{DoerrGWYAlgorithmica19,
DoerrDoerrKoetzingAlgorithmica18, DoerrEtAlTheoryGuidedBenchmark18,
DoerrWagnerPPSN18,RodionovaABDGECCO19,
FajardoGECCO19,RajabiWittGECCO20, 
LassigSudholtFOGA11}.  See also the recent 
survey article \citep{DoerrDoerrParameterBookChapter} for an in-depth coverage of parameter 
control, self-adjusting algorithms, and theoretical runtime results.

We call our algorithm the \emph{self-adjusting \ooea with asymmetric mutation} 
(\asymsa) and conduct a rigorous runtime analysis on the $\onemax_a$ 
problem. Since the above-mentioned parameter~$r$ also determines the 
expected number of flipping bits of any type, we allow the bias of 
zero- resp.\ one-bits only to change by a small constant. 
Nevertheless, we can prove a speed-up on the \onemax function with
target $1^n$ (and analogously for \zeromax with target~$0^n$) by a
factor of at least~$2$ compared to the previous asymmetric \ooea since 
the bias is adjusted towards the ``right'' type of bits. On
general $\onemax_a$, where $a$ both contains zeroes and ones, we prove
that the bias is not strongly adjusted to one type of bits such
that we recover the same asymptotic bound 
$O(n\log(\min\{\zeros{a},\ones{a}\}))$ as in \cite{JansenSudholtECJ10}. 
These results represent, to the best of our knowledge, the first runtime analysis 
of a self-adjusting asymmetric EA and pave the way for the study of more advanced such  operators.

This paper is structured as follows: in Section~\ref{sec:prel}, we introduce the studied 
algorithms and fitness function. In Section~\ref{sec:onemax-ones}, we prove for the function \onemax with the all-ones string as 
target that the \asymsa is by a factor of roughly~$2$ faster than the original asymmetric \ooea 
from \cite{JansenSudholtECJ10}. For targets $a$ containing both zero- and one-bits, we 
show in Section~\ref{sec:onemax-mixed} that our algorithm is asymptotically not slowed down 
and in fact insensitive to the second parameter~$\alpha$ of the algorithm that controls its learning speed. Experiments 
in Section~\ref{sec:experiments} demonstrate that the speedup by a factor of at least~$2$ 
already can be observed for small problem sizes while our algorithm is not noticeably slower 
than the original asymmetric one on the mixed target with half ones and zeros. We finish with 
some conclusions.

\section{Preliminaries}
\label{sec:prel}

\subsection{Algorithms}
We consider asymmetric mutation in the context of a \ooea
for the maximization of pseudo-boolean functions
$f\colon\{0,1\}^n\to \mathbb{R}$, which is arguably the most commonly investigated setting in the runtime analysis 
of EAs. The framework is given in Algorithm~\ref{alg:oneone-classic}, where we consider the following two choices for the mutation operator:
\begin{description}
    \item[Standard bit mutation]Flip each bit in a copy of $x$
    independently with probability~$\frac{1}{n}$.
    \item[Asymmetric mutation]Flip each $0$-bit in a copy of~$x$ with probability $\frac{1}{2\zeros{x}}$ and 
    each $1$-bit with probability $\frac{1}{2\ones{x}}$ (independently for all bits).
\end{description}

With standard bit mutation, we call the algorithm the \emph{classical \ooea} and the
one with asymmetric mutation the 
(static) \emph{asymmetric \ooea}. The latter algorithm stems 
from \cite{JansenSudholtECJ10} and differs from the one 
from \cite{NeumannWegenerTCS07} by introducing the
two factors~$1/2$  to avoid mutation probabilities above~$1/2$.

\begin{algorithm}[ht]
	\caption{\ooea}
	\label{alg:oneone-classic}
	\begin{algorithmic}
		\State Select $x$ uniformly at random from $\{0, 1\}^n$
		\For{$t \gets 1, 2, \dots$}
		\State Create $y$ by applying \emph{a mutation operator} to~$x$.
		\If{$f(y) \ge f(x)$}
		\State $x \gets y$.
		\EndIf
		\EndFor
	\end{algorithmic}
\end{algorithm}

The \emph{runtime} (synonymously, \emph{optimization time}), 
of Algorithm~\ref{alg:oneone-classic}, is the smallest~$t$ 
where a search point of maximal fitness (\ie, $f$\nobreakdash-value) has been created; 
this coincides with the number of $f$\nobreakdash-evaluations until that time. We are mostly interested in the expected value of the runtime  and call this the \emph{expected runtime}.

As motivated above, we investigate a biased mutation operator that can vary the individual probabilities of flipping $0$- and 
$1$\nobreakdash-bits. We propose a rather conservative setting that flips $0$\nobreakdash-bits of the underlying string~$x$ with probability $\frac{r_0}{\zeros{x}}$ for some value 
$0\le r_0\le 1$ and $1$\nobreakdash-bits with probability $\frac{1-r_0}{\ones{x}}$; hence increasing the probability 
of one type decreases the probability of the other type. Formally, we also allow other ranges $0\le r_0\le r$ for some 
$r\ge 1$, which could make sense for problems where more 
than one bit must flip for an improvement. However, in the 
context of this paper, we fix $r=1$. To ease notation, we 
use $r_1\coloneqq r-r_0$, \ie, $r_1=1-r_0$. Since $r_i$ describes the expected number of flipping $i$-bits, $i\in\{0,1\}$, we call $r_0$ the $0$\nobreakdash-strength and 
accordingly $r_1$ the $1$\nobreakdash-strength.

We now formally define the self-adjusting \ooea with asymmetric mutation (\asymsa), see Algorithm~\ref{algo:asymsa}.
Initially, $r_0=r_1=\frac{1}{2}$ so that the operator 
coincides with the static asymmetric mutation. In an observation phase of length~$N$, which 
is a parameter to be chosen beforehand, two different 
pairs of $0$- and 
$1$-strengths are tried alternatingly, which differ 
from each other by an additive term of~$2\alpha$ that controls 
the speed of change. The 
pair that 
leads to relatively more successes (\ie, strict improvements) is used as the ground   
pair for the next phase. In case of a tie (including 
a completely unsuccessful phase) a uniform random choice 
is made. Also, we ensure that the strengths are confined 
to $[\alpha,1-\alpha]$. 
Hereinafter, we say that we 
\emph{apply asymmetric mutation with probability pair $(p_0,p_1)$}
if we flip every $0$-bit of the underlying string~$x$ with 
probability $p_0$  and 
every $1$-bit with probability $p_1$ (independently for all
bits). Note that the asymmetric \ooea from \cite{JansenSudholtECJ10}
applies 
the probability pair $(\frac{1}{2\zeros{x}}, \frac{1}{2\ones{x}})$. 

\begin{algorithm}[ht]
	\caption{Self-adjusting (1+1)~EA with Asymmetric Mutation (\asymsa); parameters: $N$ = length of observation phase, $\alpha$ = learning rate} \label{algo:asymsa}
	\begin{algorithmic}
		\State Select $x$ uniformly at random from $\{0, 1\}^n$.
		\State $r\gets 1$. // parameter $r$ fixed to $1$ in this paper
		\State $r_0\gets \frac 12$, $r_1\gets r-r_0$.
		\State $b\gets 0$.
		\For{$t \gets 1, 2, \dots$}
		\State $p_-\gets (\frac{r_0-\alpha}{\zeros{x}},\frac{r_1+\alpha}{\ones{x}})$, $p_+\gets (\frac{r_0+\alpha}{\zeros{x}},\frac{r_1-\alpha}{\ones{x}})$.
		\State Create $y$ by applying asymmetric mutation with pair~$p_-$ if $t$ is odd and with pair~$p_+$ otherwise.
		\If{ $f(y)\ge f(x)$ }
		\State $x \gets y$.
		    \If{ $f(y)>  f(x)$ }
		    \If{$t$ is odd }
		    \State $b = b -1$.
		    \Else
		    \State $b = b +1$.
		    \EndIf
		\EndIf
		\EndIf
		\If{ $t \equiv 0 \pmod{N}$ }
		\If{ $b < 0$ }
		\State Replace $r_0$ with $\max\{r_0-\alpha, 2\alpha\}$ and $r_1$ with $\min\{r_1+\alpha,1-2\alpha\}$.
		\ElsIf{$b > 0$}
		\State Replace $r_0$ with $\min\{r_0+\alpha, 1-2\alpha\}$ and $r_1$ with $\max\{r_1-\alpha,2\alpha\}$.
		\Else
		 \State Perform one of the following two actions with prob.~$1/2$:
		 \State \quad -- Replace $r_0$ with $\max\{r_0-\alpha, 2\alpha\}$ and $r_1$ with $\min\{r_1+\alpha,1-2\alpha\}$.
		 \State \quad -- Replace $r_0$ with $\min\{r_0+\alpha, 1-2\alpha\}$ and $r_1$ with $\max\{r_1-\alpha,2\alpha\}$.
		\EndIf
		\State $b \gets 0$.
		\EndIf
		\EndFor
	\end{algorithmic}
\end{algorithm}

So far, apart from the obvious restriction $\alpha<1/4$, the choice of the parameter $\alpha$ is open. 
We mostly set it to small constant values but also allow it to converge
to~$0$ slowly. In experiments, we set the parameter $\alpha$ to $0.1$. Note that if the extreme pair of strengths $(\alpha,1-\alpha)$ 
is used 
then $0$-bits are only rarely flipped and $1$-bits with roughly twice the probability of the static asymmetric operator.

In this paper, we are exclusively concerned with the maximization of the pseudo-Boolean function
\[
\onemax_a(x) \coloneqq n-H(x,a)
\]
for an unknown target~$a\in\{0,1\}^n$, where $H(\cdot,\cdot)$ denotes the Hamming distance. Hence, 
$\onemax_a(x)$ returns the number of matching bits in~$x$ and~$a$. In unbiased
algorithms, the target~$a$ is usually and without loss of generality assumed as the 
all-ones string~$1^n$, so that we denote $\onemax=\onemax_{1^n}$. In the considered asymmetric setting,
the choice~$a$ makes a difference. Note, however, that only the number $\ones{a}$ influences 
the runtime behavior of the considered asymmetric algorithms and not the absolute positions of 
these $\ones{a}$ $1$\nobreakdash-bits.

\section{Analysis of Self-adjusting Asymmetric Mutation on \onemax}
\label{sec:onemax-ones}

In this section, we show in Theorem~\ref{theo:onemax} that the \asymsa is 
by a factor of roughly~$2$ faster on 
\onemax, \ie, when the target $a$ is the all-ones string, than the static asymmetric
\ooea from \cite{JansenSudholtECJ10}. The proof shows that the self-adjusting scheme 
is likely to set the $0$-strength to its maximum $1-2\alpha$ 
which makes improvements more 
likely than with the initial strength of~$1/2$. On the way to the proof, we 
develop two helper results, stated in Lemmas~\ref{lem:diff-beta} 
and~\ref{lem:direction-probability} below.

In the following lemma, we show that the algorithm is likely to observe more improvements with larger $0$-strength and smaller $1$-strength on \onemax.

\begin{lemma} \label{lem:diff-beta}
     Consider the \asymsa on $\onemax=\onemax_{1^n}$ and let~$x$ denote its current search point.
     For $\beta\geq 0$ we have
    \[\prob{S_{(\frac{r_0+\beta}{\zeros{x}},\frac{r_1-\beta}{\ones{x}})}} \geq \prob{S_{(\frac{r_0}{\zeros{x}},\frac{r_1}{\ones{x}})}} + r_1r_0 (1-e^{-\beta}),\]
    where $S_{(p_0,p_1)}$ is the event of observing a strict improvement when in each iteration zero-bits and one-bits are flipped with probability $p_0$ and $p_1$ respectively.
\end{lemma}

\begin{proof}
Consider the following random experiment which is formally known as a coupling (see \cite{DoerrProbabilisticBookChapter} for more information on this concept). Assume that we flip bits in two phases: in the first phase, we flip all one-bits with probability~$(r_1-\beta)/\ones{x}$ and all zero-bits with probability~$r_0/\zeros{x}$. In the second phase, we only flip zero-bits with probability~$\beta/(\zeros{x}-r_0)$. We fixed $\zeros{x}$ and $\ones{x}$ before the experiment.
At the end of the second phase, the probability of a bit being flipped equals $(r_1-\beta)/\ones{x}$ if the value of the bit is one and $(r_0+\beta)/\zeros{x}$ if the value of the bit is zero. The former is trivial (the bit is flipped only once with that probability) but for the latter, the probability of flipping each zero-bit equals $1-(1-r_0/\zeros{x})(1-\beta/(\zeros{x}-r_0)) =  (r_0+\beta)/\zeros{x}$. Therefore, the probability of observing an improvement at the end of the second phase is 
\begin{align}
\prob{S_{(\frac{r_0+\beta}{\zeros{x}},\frac{r_1-\beta}{\ones{x}})}}. \label{eq:success-direct}
\end{align}
Let us now calculate the probability of observing an improvement differently by computing the probability of success in each phase separately. At the end of the first phase, the probability of an improvement is $\probtext{S_{(\frac{r_0}{\zeros{x}},\frac{r_1-\beta}{\ones{x}})}}.$
It fails with probability~$1-\probtext{S_{(\frac{r_0}{\zeros{x}},\frac{r_1-\beta}{\ones{x}})}}$ and the probability of success in the second phase is at least $\left(1-(r_1-\beta)/\ones{x}\right)^{\ones{x}} \probtext{S_{(\frac{\beta}{\zeros{x}-r_0},0)}}, $ where no one-bits are flipped in the first phase and the algorithm finds a better search point when it only flips zero-bits with probability $\beta/(\zeros{x}-r_0)$. All together we can say that the probability of observing a success is at least
\small{
\begin{align}
    \prob{S_{(\frac{r_0}{\zeros{x}},\frac{r_1-\beta}{\ones{x}})}}
     + 
    \left(1-\prob{S_{(\frac{r_0}{\zeros{x}},\frac{r_1-\beta}{\ones{x}})}}\right)
    \left(1-\frac{r_1-\beta}{\ones{x}}\right)^{\ones{x}} \prob{S_{(\frac{\beta}{\zeros{x}-r_0},0)}}. \label{eq:success-indirect}
\end{align}}

By considering (\ref{eq:success-direct}) and (\ref{eq:success-indirect}), we obtain

\begin{align*}
    \prob{S_{(\frac{r_0+\beta}{\zeros{x}},\frac{r_1-\beta}{\ones{x}})}} &\geq \prob{S_{(\frac{r_0}{\zeros{x}},\frac{r_1-\beta}{\ones{x}})}} \\ 
    & + 
    \left(1-\prob{S_{(\frac{r_0}{\zeros{x}},\frac{r_1-\beta}{\ones{x}})}}\right)
    \left(1-\frac{r_1-\beta}{\ones{x}}\right)^{\ones{x}} \prob{S_{(\frac{\beta}{\zeros{x}-r_0},0)}}.
\end{align*}
Note that $\probtext{S_{(\frac{r_0}{\zeros{x}},\frac{r_1-\beta}{\ones{x}})}}\le (r_0/\zeros{x}) \zeros{x}=r_0$. Also, we have $\probtext{S_{(\frac{r_0}{\zeros{x}},\frac{r_1-\beta}{\ones{x}})}} \geq \probtext{S_{(\frac{r_0}{\zeros{x}},\frac{r_1}{\ones{x}})}}$ since in the second setting, the probability of flipping one-bits is larger, with the same probability of flipping zero-bits, which results in flipping more one-bits.

Finally, we have
\begin{align*}
    \prob{S_{(\frac{r_0+\beta}{\zeros{x}},\frac{r_1-\beta}{\ones{x}})}} &\geq \prob{S_{(\frac{r_0}{\zeros{x}},\frac{r_1-\beta}{\ones{x}})}} \\ 
    & + 
    \left(1-\prob{S_{(\frac{r_0}{\zeros{x}},\frac{r_1-\beta}{\ones{x}})}}\right)
    \left(1-\frac{r_1-\beta}{\ones{x}}\right)^{\ones{x}} \prob{S_{(\frac{\beta}{\zeros{x}-r_0},0)}} \\
    & \geq \prob{S_{(\frac{r_0}{\zeros{x}},\frac{r_1}{\ones{x}})}} + (1-r_0)(1-r_1+\beta) \left(
    1-(1-\frac{\beta}{\zeros{x}})^{\zeros{x}}\right) \\
    & \geq \prob{S_{(\frac{r_0}{\zeros{x}},\frac{r_1}{\ones{x}})}} + r_1r_0 (1-e^{-\beta}),
\end{align*}
since $r_1+r_0=1$ and $\beta \ge 0$ as well as for $t\ge 1$ and $0\le s \le t$, we have the inequality $1-s\le (1-s/t)^t\le e^{-s}$.
\end{proof}

We can apply concentration inequalities to show that the larger 
success probability with the pair $p_+$ with high probability makes the 
algorithm to move to the larger $0$-strength at the end of an observation phase. This requires a careful analysis since the success probabilities themselves change in the observation phase of length~$N$. 

\begin{lemma} \label{lem:direction-probability}
Assume that $\min\{\zeros{x},\ones{x}\}=\Omega(n^{5/6})$ for all search points $x$ in an observation phase of length~$N$ of the \asymsa on \onemax. With probability at least $1-\epsilon$, if $\alpha=\omega(n^{-1/12}) \cap (0,1/4)$, and $N\ge 8 \alpha^{-8} \ln(4/\epsilon)$, for an arbitrarily small constant $\epsilon\in (0,1/2)$, the algorithm chooses the probability pair $p_+$ on at the end of the phase.
\end{lemma}

\begin{proof}
When the algorithm accepts a new search point, the number of one-bits and zero-bits of the current search point are changed, which results in modified success probabilities for both settings $p_+$ and $p_-$. More concretely, if $\zeros{x}$ is the number of $0$\nobreakdash-bits at the beginning of the
observation phase of length~$N$ and the \onemax-value 
increases by~$g>0$ in the phase, then the $0$-strength of 
the pair $p_+$ increases from $\frac{r_0+\alpha}{\zeros{x}}$ 
to $\frac{r_0+\alpha}{\zeros{x}-g}$ and analogously for the
$1$\nobreakdash-strength, which decreases from 
$\frac{r_1-\alpha}{\ones{x}}$ to $\frac{r_1-\alpha}{\ones{x}+g}$. Analogous considerations hold for the change of strengths 
with respect to the pair $p_-$. Clearly, the $0$\nobreakdash-strength 
and $1$\nobreakdash-strength in the pair $p_+$  change in a beneficial 
 way, and we 
can pessimistically assume that theses components  
stay at the values $\frac{r_0+\alpha}{\zeros{x}}$ and 
$\frac{r_1-\alpha}{\ones{x}}$, respectively, used at the 
beginning of the phase. However, the opposite is true for 
the two strengths in $p_-$.

Our aim is to show that, under some boundedness assumption on~$g$, we have 
\begin{equation}
\frac{r_0-\alpha}{\zeros{x}-g} \le \frac{r_0-\alpha/2}{\zeros{x}}
\label{eq:r0-change}
\end{equation}
and 
\begin{equation}
\frac{r_1+\alpha}{\ones{x}+g} \ge \frac{r_1+\alpha/2}{\ones{x}}
\label{eq:r1-change}
\end{equation}
so that we can apply with 
 Lemma~\ref{lem:diff-beta} with $\beta=(3/2)\alpha$. 
 
 To show the two inequalities, we first restate them 
 equivalently as 
\begin{equation}
\frac{r_0-\alpha}{r_0-\alpha/2} \le \frac{\zeros{x}-g}{\zeros{x}}
\label{eq:r0-change-alt}
\end{equation}
and 
\begin{equation}
\frac{r_1+\alpha}{r_1+\alpha/2} \ge \frac{\ones{x}+g}{\ones{x}}
\label{eq:r1-change-alt}
\end{equation}
and observe that sufficient conditions for them to hold 
are given by 
\begin{equation}
\frac{1}{1+\alpha/2} \le \frac{\zeros{x}-g}{\zeros{x}}
\label{eq:r0-change-alt-suff}
\end{equation}
and 
\begin{equation}
\frac{1}{1-\alpha/2} \ge \frac{\ones{x}+g}{\ones{x}}
\label{eq:r1-change-alt-suff}
\end{equation}
by increasing $r_0$ and $r_1$ to $1+\alpha$ and $1-\alpha$, respectively, which are already at or above the maximum values
for these strengths.

Then we 
recall our assumption that
 $\min\{\ones{x},\zeros{x}\}\ge n^{5/6}$ at the beginning of 
 the phase. Since $\alpha=\omega(n^{-1/12})$, we have 
 $N=O(\alpha^8) = O(n^{3/4})$. Since each step of the phase 
 increases the number of $1$-bits by an expected value of
 at most~$1$, we have an expected increase of at most~$N$, 
 and by Markov's inequality, the increase is bounded from 
 above by $g\coloneqq \frac{2}{\epsilon}N=O(n^{3/4})$ with 
 probability at least $1-\epsilon/2$ (recalling that 
 $\epsilon$ is constant). We assume this bound to hold 
 in the following. Plugging $g$ and our bounds 
 on $\ones{x}$ and $\zeros{x}$ in, we obtain
 \[
\frac{\zeros{x}-g}{\zeros{x}} \ge \frac{n^{5/6}-O(n^{3/4})}{n^{5/6}} = 1-O(n^{-1/12}) \ge \frac{1}{1+\alpha/2}
\]
for $n$ sufficiently large since $\alpha=\omega(n^{-1/12})$. 
This establishes \eqref{eq:r0-change-alt-suff}, implying 
\eqref{eq:r0-change-alt} and thereby \eqref{eq:r0-change}. 
With a completely analogous procedure, we also
prove~\eqref{eq:r1-change-alt-suff} and 
finally~\eqref{eq:r1-change}.

With these pessimistic $0$- and $1$-strengths holding 
throughout the phase, 
let us define $s_+$ and $s_-$ as $\probtext{S_{(\frac{r_0+\alpha}{\zeros{x}},\frac{r_1-\alpha/2}{\ones{x}})}}$ and $\probtext{S_{(\frac{r_0-\alpha/2}{\zeros{x}},\frac{r_1+\alpha}{\ones{x}})}}$ respectively. Also, let $X_+$ and $X_-$ be the random variables showing the number of improvements during a phase containing $N$ iterations for the (pessimistic) settings $p_+$ and $p_-$ respectively. Since in each iteration of a phase, $s_+$ and $s_-$ are lower bounds on the actual 
success probabilities, both $X_+$ and $X_-$ are in fact 
pessimistically estimated by a sum of some independent and identically distributed random variables and Chernoff bounds 
can be applied.

Applying Lemma~\ref{lem:diff-beta} with $\beta=(3/2)\alpha$, we have 
\begin{align*}
    s_+-s_- \ge (r_0-\alpha)(r_1+\alpha)(1-e^{-\frac32\alpha}) \ge \alpha(1-\alpha) \frac32\alpha(1-\frac32\alpha)\ge \frac 32\alpha^4,
\end{align*}
since we have $(r_0-\alpha)(r_1+\alpha)\ge \alpha(1-\alpha)$ because $0$-strengths and $1$-strengths are only chosen from range $\left[\alpha,1-\alpha\right]$, and additionally the sum of the strengths equals one. Also, via Lemma 1.4.2(b)
in~\cite{DoerrProbabilisticBookChapter} we have $(1-e^{-3/2\alpha})\ge 3/2\alpha-9/4\alpha^2$. Lastly, $\alpha<1-\alpha$ and $\alpha<1-\frac 32\alpha$ since $\alpha<0.25$.

In order to find a lower bound for $X_+$, we have $E[X_+]=s_+N/2$ since we use setting $p_+$ for the half of iterations. Through Chernoff, we calculate
\begin{align*}
    \prob{X_+\le \frac{s_++s_-}{4} N} & =
    \prob{X_+\leq \left(1-\left(1-\frac{s_++s_-}{2s_+}\right)\right)\frac N2 s_+} \\
    & \le  \exp\left(-\frac{(1-\frac{s_++s_-}{2s_+})^2}{2}\cdot \frac N2 s_+\right) \\
    & \le  \exp\left(-\frac{(s_+-s_-)^2}{16s_+} N\right) 
     <  \frac{\epsilon}4.
\end{align*}

Similarly, we have $E[X_-]=s_-N/2$. In order to compute $\prob{X_-\ge \frac{s_++s_-}{4}N}$, we use Chernoff bounds with $\delta=\min\{(s_+-s_-)/2s_-, 1/2\}$, which results in the following lower bound
\begin{align*}
    \prob{X_-\ge \frac{s_-+s_+}{4} N)} &\le
    \prob{X_-\ge (1+\delta)\frac N2 s_-} \\
    &\le  \exp\left(-\frac{\delta^2}{3}\cdot \frac N2 s_-\right) 
    \le  \frac{\epsilon}4.
\end{align*}

Finally, using the obtained bounds, we compute
\begin{align*}
    \prob{X_+\le X_-} &\leq \prob{X_+\le \frac{s_-+s_+}{4}N)} + \prob{X_-\ge \frac{s_-+s_+}{4}N)} \\
    &< \epsilon/4 + \epsilon/4 = \epsilon/2.
\end{align*}
Since our bound~$g$ on the progress in the phase holds 
with probability at least 
$1-\epsilon/2$, the total success probability
is at least $1-\epsilon$.
\end{proof}

We can now state and prove our main 
result from this section.

\begin{theorem} \label{theo:onemax}
The expected optimization time $E(T)$ of the \asymsa with $\alpha=\omega(n^{-1/12}) \cap (0,1/4)$, $N \ge  \lceil 8 \alpha^{-8} \ln(4/\epsilon)\rceil$, for an arbitrarily small constant $\epsilon\in (0,1/2)$, and $N=o(n)$ on \onemax and \zeromax satisfies
\[\frac n2(1-\alpha)^{-1} \le E(T) \le \frac n2(1-4\alpha)^{-1} + o(n).\]
\end{theorem}

This theorem rigorously shows that the presented algorithm outperforms the original asymmetric \ooea from \cite{JansenSudholtECJ10} by a constant factor of roughly 2 on \onemax. Indeed, the number of iterations needed for the  asymmetric \ooea is at least $n$ since by drift analysis \citep{LenglerDriftBookChapter}, the drift is at most $\zeros{x}\cdot 1/(2\zeros{x})=1/2$ (recalling that each $0$-bit flips with probability $1/(2\zeros{x})$) and $E[H(x^*,1^n)]=n/2$ for the initial search point~$x^*$  while Theorem~\ref{theo:onemax} proves that the \asymsa finds the optimum point within $(1-\Theta(\alpha))^{-1}n/2 + o(n)$ iterations. We even believe that the speedup of our algorithm is close to $2e^{1/2}$; namely, the original asymmetric \ooea in an improvement usually must preserve existing one-bits, which happens with probability $(1-1/(2\ones{x}))^{\ones{x}}\approx e^{-1/2}$. This probability is close to $1-2\alpha$ in our algorithm when the self-adaptation has increased the $0$\nobreakdash-strength to $1-2\alpha$ and thereby decreased the $1$\nobreakdash-strength to~$2\alpha$.

\begin{proofof}{Theorem~\ref{theo:onemax}}
By Chernoff's bound, we have $H(x^*, 1^n) \le n/2+n^{3/4}$ with probability $1-e^{-\Omega(n)}$, where $x^*$ is the initial search point. Having the elitist selection mechanism results in keeping that condition for all future search points. 

Now, we divide the run of the algorithm into two epochs according to the number of zero-bits. In the first epoch, we assume that $\zeros{x}>n^{5/6}$. Since $\ones{x}>n^{5/6}$ also holds (proved by Chernoff), according to Lemma~\ref{lem:direction-probability} with $\epsilon<1/2$, the algorithm sets $p_+$ as its current probability pair in each observation phase with probability at least
$1-\epsilon$. Hence, we have a random walk so the number of phases to reach $p^* \coloneqq ((1-\alpha)/\zeros{x},\alpha/\ones{x})$, by using Gambler's Ruin \cite{Feller1}, 
is at most $(1/(2\alpha))/(1-\epsilon-\epsilon)$. Since each phase contains $N$ iterations, the expected number of iterations to reach $p^*$ for the first time is at most $ N/(2\alpha(1-2\epsilon))$.

When $p_+$ is equal to $p^*$, the algorithm uses the probability pairs \[
p_+=\left(\frac{1-\alpha}{\zeros{x}},\frac{\alpha}{\ones{x}}\right) \text{ and }p_-=\left(\frac{1-2\alpha}{\zeros{x}},\frac{2\alpha}{\ones{x}}\right)
\] in the iterations so the drift is greater than 
  \begin{align*}
s \frac{1-2\alpha}{s} \left(1-\frac{2\alpha}{n-s}\right)^{n-s}\ge (1-2\alpha)^2\geq 1-4\alpha.
  \end{align*}
  
  Consequently, via additive drift theorem \citep{LenglerDriftBookChapter}, we need at most
  $(n/2)(1-4\alpha)^{-1} + o(n)$ iterations where $p_+$ equals $p^*$.
  
   After the first time where $p_+$ gets equal to $p^*$, there is a possibility  of getting $p_+$ away from $p^*$. With the probability of $\epsilon$, the algorithm chooses probability pair $p_-$. The expected number of phases to reach $p^*$ again is $(1-2\epsilon)^{-1}$ from Gambler's Ruin \cite{Feller1}, resulting in $N \epsilon  (1-2\epsilon)^{-1}$ extra iterations for each phase where $p^*$ is~$p_+$. The expected number of steps needed until we have one step with pair~$p^*$ is $\epsilon(1-2\epsilon)^{-1}$ and by linearity of expectation this factor also holds for the expected number of such steps in the drift analysis.
  
  Overall, the expected number of iterations of the first epoch is at most 
  \[ \frac{n}{2}(1-4\alpha)^{-1}\frac{\epsilon}{1-2\epsilon}+\frac{N}{2\alpha(1-2\epsilon)}+o(n)
  = \frac{n}{2}(1-4\alpha)^{-1}\frac{\epsilon}{1-2\epsilon} 
  + O(\alpha^{-9})+o(n)
  ,\]
  where we used $N=o(n)$.

In the second epoch where $\zeros{x}\le n^{5/6}$, the expected number of iterations is at most $O(n^{5/6}/\alpha)$ since in the worst case with the probability pair $p_-=(\alpha/\zeros{x},(1-\alpha)/\ones{x})$ the drift is $(s\alpha)/s(1-(1-\alpha)/(n-s))^{n-s}\geq 2\alpha$.

All together and by our assumption that $\alpha=\omega(n^{-1/12}) \cap (0,1/4)$, we have
\[
E[T] \leq (1-4\alpha)^{-1}n/2+o(n).
\]
  
Moreover, in order to compute the lower bound, we have the upper bound $s(1-\alpha)/s = 1-\alpha$ on the drift, so through additive drift theorem \citep{LenglerDriftBookChapter}, we have \[E[T] \geq (1-\alpha)^{-1}n/2.\] 
\end{proofof}

\section{Different Targets than All-Ones and All-Zeros}
\label{sec:onemax-mixed}
After we have seen that the self-adjusting asymmetry is beneficial on the usual \onemax and \zeromax function,
we now demonstrate that this  self-adaptation  does not considerably harm 
its optimization time on $\onemax_a$ for targets 
$a$ different from the all-ones and all-zeros string. 
A worst case bound, obtained by assuming a pessimistic strength of $\alpha$ for the bits that still have to 
be flipped, would be $O((n/\alpha)\log z)$ following the ideas 
from \cite{JansenSudholtECJ10}. We show that the factor $1/\alpha$
stemming from this pessimistic assumption does not actually appear since 
it becomes rare that the strength takes such a worst-case value on $\onemax_a$ 
if $a$ contains both many zeros and ones. Hence, we obtain the same asymptotic bound 
as for the original asymmetric \ooea in this case.

\begin{theorem}
Assume $z=\min\{\ones{a}, \zeros{a}\}=\omega(\ln n)$, $1/\alpha=O(\log z/\!\log\log n)$ and $N=O(1)$. Then the expected optimization
time of the \asymsa  on $\onemax_a$ is $O(n\log z)$,  where the implicit constant in the $O$ does not depend on $\alpha$.
\end{theorem}

\begin{proof}
\Wlo{} $\ones{a}\le n/2$ and $a$ has the form $0^z 1^{n-z}$, where
we refer to the first $z$  bits as the prefix and the final $n-z$ bits as the suffix.

We divide the run of the \asymsa into three epochs according to the Hamming distance $h_t=n-H(x_t,a)$ 
of the current search point at time~$t$ to~$a$, where the division 
is partially inspired by the analysis in \cite{JansenSudholtECJ10}. If $h_t\ge 2z$ there
are at least $h_t-z\ge h_t/2$ $0$\nobreakdash-bits in the suffix that can be flipped to improve and we 
also have $\zeros{x_t} \le h_t+z\le (3/2)h_t$.  Then
the probability of an improvement is at least $((h_t/2)(\alpha/\zeros{x})) (1-\ones{x}/n)^{\ones{x}} \ge (\alpha/3) e^{-1-o(1)}$.  Hence, the expected length of the first epoch is $O(n/\alpha)=O(n\log z)$ since $1/\alpha=O(\log z/\!\log\log n)$.

In the second epoch, we have $h_t \le 2z$. This epoch ends when $h_t \le n/\!\log^{3} n$ and may therefore be empty. 
Since nothing is to show otherwise, we assume $z\ge n/\!\log^{3} n$ in this part of the analysis.
Since the probability 
of flipping an incorrect bit and not 
flipping any other bit is always at least $(\alpha/n)e^{-1+o(1)}$ 
the expected length of this epoch is
at most 
$
\sum_{i=n/\!\log^{3} n}^z \frac{e^2 n}{\alpha i} = O((n/\alpha)\ln\ln n), $
which is $O(n\ln z)$ since $1/\alpha=O(\log z/\!\log\log n)$.

At the start of the third epoch, we have $h_t \le z^*$, where $z^*=\min\{2z,n/\!\log^{3} n\}$. The epoch ends when the optimum is found and is divided into $z^*$ phases of varying length.
Phase $i$, where $1\le i\le z^*$, starts when the Hamming distance to~$a$ has become~$i$ and ends before the step where the distance becomes strictly less than $i$; the phase may be empty. The aim is to show that the expected length 
of phase $i$ is $O(n/i)$ independently of $\alpha$. To this end, we concretely consider a phase of length $cn/i$ for a sufficiently large constant $c$ and divide 
it into subphases of length $c_1/\alpha^2$ for a sufficiently large constant $c_1$. The crucial observation we will prove is that such a subphase with 
probability~$\Omega(1)$ contains at least $c_2/\alpha^2$ steps such that the $r_0$-strength is in the interval $[1/4,3/4]$, with $c_2$ being another 
sufficiently large constant. 
During these steps the probability  of ending the phase due to an improvement is at least $c_3 i/n$ 
for some constant $c_3>0$. Hence, the probability of an improvement within 
$O(n\alpha^2/i)$ such subphases is $\Omega(1)$, proving that the expected number of phases of length~$cn/i$ is $O(1)$. 
Altogether, the expected length of phase~$i$ is $O(n/i)$. Note that $i\le z^*$ and $1/\alpha=O(\log n/\log\log n)$, 
so that $n\alpha^2/i$ is asymptotically larger than~$1$. Summing over 
 $i\in\{1,\dots,z^*\}$, the expected length of the third epoch is $O(n\log z)$. 

We are left with the claim that at least $c_2/(2\alpha^2)$ steps in a subphase of length $c_1/\alpha^2$ have a $0$\nobreakdash-strength within $[1/4,3/4]$. To this end, we study 
the random process of the $0$\nobreakdash-strength and relate it to an unbiased 
Markovian random walk on the state space 
$\alpha,2\alpha,\dots,1-2\alpha,1-\alpha$, which we rescale to $1,\dots,\alpha-1$ by multiplying 
with $1/\alpha$. 
If the overarching phase of length~$cn/i$ leads to an improvement, there is nothing to show. Hence, we pessimistically assume that no improvement has been found yet in the considered sequence of subphases of length~$c_1/\alpha^2$ each.  Therefore, since the algorithm makes a uniform choice in the absence of improvements, both the probability of increasing and decreasing the $0$\nobreakdash-strength equal $1/2$. 
The borders $1$ and $\alpha-1$ of the chain are reflecting with probability $1/2$ and looping with the remaining probability.

  From any state of the Markov chain, it is sufficient to bridge a distance of $m=1/(2\alpha)$ to reach~$1/2$. Pessimistically, 
we assume strength $\alpha$ (\ie, state~$1$) at the beginning of the subphase. Using well-known results for the fair Gambler's Ruin \citep{Feller1}, the expected time to reach state~$m$ from~$X_0$ 
conditional on 
not reaching  
state~$0$ (which we here interpret as looping on state~$1$) is $X_0(m-X_0)$.
Also, the probability of reaching $m$ from $X_0$ before reaching~$0$ equals
$ 
X_0/m
$
and the expected number of repetitions until state~$m$ is reached is the reciprocal of this. 

Combining both, the expected time until $m$ is reached from~$X_0=1$ is at most 
$m^2-m$, and by Markov's inequality, the time is at most $2m^2=\alpha^{-2}$ with probability at least~$1/2$.

We next show that the time to leave the interval 
$[m/2,3m/2]$ (corresponding to strength $[1/4,3/4]$) is at least $c/\alpha^2$ with probability at least $\Omega(1)$. To this end, we apply Chernoff bounds 
to estimate the number of decreasing steps within $c/\alpha^2$ steps and note that it is less than $1/\alpha$ with probability $\Omega(1)$. 

Altogether, with probability $\Omega(1)$ a subphase of length $c_1/\alpha^2$ contains 
at least $c_2/\alpha^2$ steps with strength within $[1/4,3/4]$ as suggested.

The theorem follows by summing up the expected lengths of the  epochs.\end{proof}

\section{Experiments}
\label{sec:experiments}
The results in the previous sections are mostly asymptotic. In addition, although the obtained sufficient value for $N$ derived in Lemma~\ref{lem:diff-beta} is constant for $\alpha=\Omega(1)$, it can be large for a relatively small bit string. Hence, in this section, we present the results of the experiments conducted in order to see how the presented algorithm performs in practice.
	
	We ran an implementation of Algorithm~\ref{algo:asymsa} (\asymsa) on \onemax and $\onemax_a$ with $a=0^{n/2}1^{n/2}$ and with $n$ varying from 8000 to 20000. The selected parameters of the Algorithm \asymsa are  $\alpha=0.1$ and $N=50$.
	We compared our algorithm against the \ooea with standard mutation rate $1/n$ and asymmetric \ooea proposed in \cite{JansenSudholtECJ10}.

\begin{figure}
	\includegraphics[width=\linewidth]{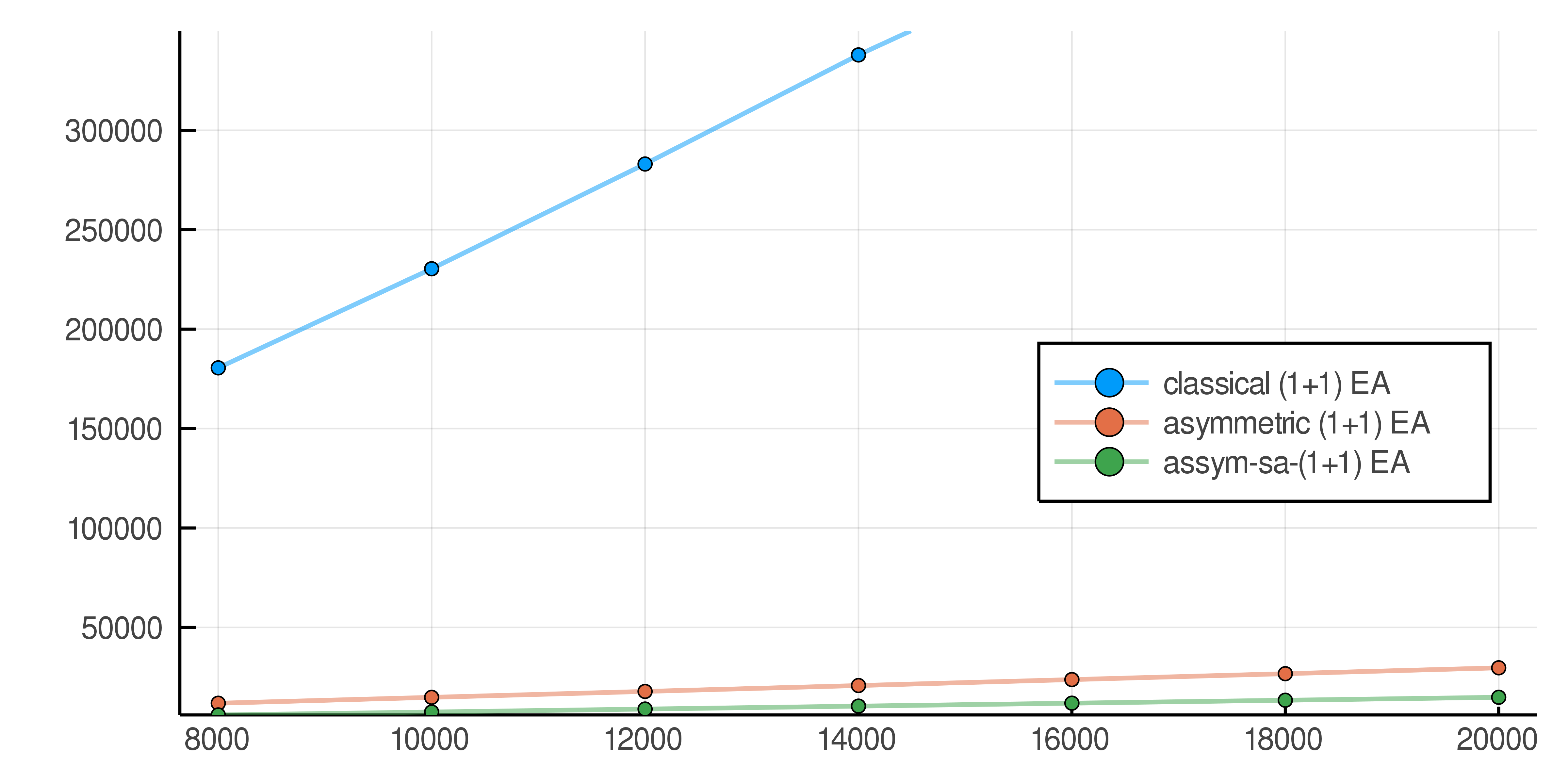}
	\caption{Average number of fitness calls (over 1000 runs) the mentioned algorithms took to optimize $\onemax$.}
	\label{fig:onemax}
\end{figure}

The average and the standard deviation of optimization time of the experiment carried out on \onemax can be seen in Figure~\ref{fig:onemax} and Table~\ref{tab:onemax}. There is a clear difference between the performance of classical \ooea and the algorithms using asymmetric mutations. It shows that these biased mutations can speed up the optimization time considerably. Also, the \asymsa outperforms asymmetric \ooea although we used relatively small $N$ and large $\alpha$. In detail, the average number of iterations that the \asymsa and asymmetric \ooea took to optimize for $n=20000$  are 14864 and 29683 respectively. By considering the standard deviation of 238 and 385 as well, we can argue that our algorithm is faster by a factor of empirically around~$2$.

\begin{table}[H]
\begin{tabularx}{\textwidth}{Y|Y|Y|Y|Y|Y|Y|}
\cline{2-7}
\multirow{2}{*}{}           & \multicolumn{2}{l|}{classical \ooea} & \multicolumn{2}{l|}{asymmetric \ooea} & \multicolumn{2}{l|}{\asymsa} \\ \cline{2-7} 
                            & $\bar{x}$               & $\sigma$     & $\bar{x}$       & $\sigma$     &  $\bar{x}$          & $\sigma$          \\ \hline
\multicolumn{1}{|l|}{n=8000}  &180516&27761.1&11875.8&240.891&5988.52&151.843\\ \hline
\multicolumn{1}{|l|}{10000} &230356&34715.8&14844.5&268.439&7464.44&170.273\\ \hline
\multicolumn{1}{|l|}{12000} &283046&42255.1&17807.3&292.445&8944.98&188.93\\ \hline
\multicolumn{1}{|l|}{14000} &337897&48529.9&20775.6&310.154&10423.6&200.662\\ \hline
\multicolumn{1}{|l|}{16000} &387250&54296.7&23765.8&324.768&11900.4&214.656\\ \hline
\multicolumn{1}{|l|}{18000} &442638&61438.8&26729.5&366.108&13380.6&225.03\\ \hline
\multicolumn{1}{|l|}{20000} &499727&65972.0&29683.7&385.152&14864.6&238.49\\ \hline
\end{tabularx}
\caption{Average number ($\bar{x}$) and standard deviation ($\sigma$) of fitness calls (over 1000 runs) the mentioned algorithms took to optimize $\onemax$.}
	\label{tab:onemax}
\end{table}

\begin{table}[ht]
\small
\centering
\begin{tabularx}{\textwidth}{|c|Y|Y|Y|Y|Y|Y|Y|}
\hline 
\diagbox{Algorithm}{$n$} & 8000   & 10000  & 12000  & 14000  & 16000  & 18000  & 20000  \\ \hline
classical \ooea & 180019 & 231404 & 283965 & 335668 & 389859 & 444544 & 498244 \\ \hline
asymmetric \ooea  & 180325 & 231999 & 283025 & 338264 & 391415 & 443910 & 502393 \\ \hline
\asymsa    & 180811 & 232412 & 284251 & 337519 & 390710 & 446396 & 503969 \\ \hline
\end{tabularx}
\vspace{1mm}
\caption{Average number of fitness calls (over 1000 runs) the mentioned algorithms took to optimize $\onemax_a$ with $a=0^{n/2}1^{n/2}$.}
\label{table:onemax_a}
\end{table}

In Table \ref{table:onemax_a}, the average numbers of iterations which are taken for the algorithms to find the optimum on $\onemax_a$ with $a=0^{n/2}1^{n/2}$ are available. The similarity between data for each bit string size suggests that the asymmetric algorithms perform neither worse nor better compared to classical \ooea. More precisely, all p-values obtained from a Mann-Whitney U test between algorithms, with respect to the null hypothesis of identical behavior, are greater than 0.1.

\section*{Conclusions}
We have designed and analyzed a \ooea with self-adjusting asymmetric mutation. The underlying mutation operator chooses $0$- and $1$\nobreakdash-bits of the current search point with 
different probabilities that can be adjusted based on the number of successes with
a given probability profile.

As a proof of concept, we analyzed this algorithm on instances from the function class $\onemax_a$ describing the number of matching bits with a target~$a\in\{0,1\}^n$. A rigorous runtime analysis shows that on the usual \onemax function with target $1^n$ (and analogously for target~$0^n$), the asymmetry of the operator is adjusted in a beneficial way, leading to a constant-factor speedup compared to the asymmetric \ooea without self-adjustment from~\cite{JansenSudholtECJ10}. For different targets~$a$, the asymmetry of our operator does not become considerably pronounced so that the self-adjusting scheme asymptotically does not slow down the algorithm. Experiments confirm that 
our algorithm is faster than the static asymmetric 
variant on \onemax with target~$1^n$ and not considerably 
slower for a target with equal number of $0$- and $1$\nobreakdash-bits.

An obvious topic for future research is to consider other ranges for the parameter~$r$ that determines the maximum degree of asymmetry chosen by the algorithm. In the present framework, this parameter is linked to the expected number of flipping bits (regardless of their value), so that it is usually not advisable to increase it beyond constant values. Instead, we plan to investigate settings with two parameters where $0$- and $1$-bits each have 
their self-adjusted mutation strengths. 

\section*{Acknowledgement}
	This work was supported  by a grant by the Danish Council for Independent Research  (DFF-FNU  8021-00260B).

\bibliographystyle{mynatbib_english}

\bibliography{references}

\end{document}